\documentclass{article}

\usepackage{microtype}
\usepackage{graphicx}
\usepackage{subcaption}
\usepackage{booktabs}
\usepackage{hyperref}

\usepackage[accepted]{icml2026}

\usepackage{amsmath}
\usepackage{amssymb}
\usepackage{mathtools}
\usepackage{amsthm}
\newtheorem{proposition}{Proposition}
\usepackage{multirow}
\usepackage{enumitem}

\icmltitlerunning{Behavioral Consistency in LLM-Based Agents}

\begin{document}

\twocolumn[
  \icmltitle{When Agents Disagree With Themselves: \\ Behavioral Consistency as an Uncertainty Signal for LLM Agents}

  \begin{icmlauthorlist}
    \icmlauthor{Aman Mehta}{sf}
  \end{icmlauthorlist}

  \icmlaffiliation{sf}{Snowflake AI Research}

  \icmlcorrespondingauthor{Aman Mehta}{aman.mehta@snowflake.com}

  \vskip 0.3in
]

\printAffiliationsAndNotice{}

\begin{abstract}
Running the same LLM agent on identical inputs yields 2.3--4.2 distinct action sequences per 10 runs; this behavioral variance constitutes a training-free, black-box uncertainty signal that instantiates selective classification and distribution-free calibration for agentic systems.
Across 8,000 runs of four models on 200 HotpotQA questions, consistent tasks ($\leq$2 unique paths) achieve 82--87\% accuracy while inconsistent tasks ($\geq$4 paths) achieve 41--65\%, a gap that survives controls for task difficulty.
Divergence concentrates at step~2 (50.5\% of Llama tasks), and consistency metrics detect failures with AUROC 0.62--0.78.
Exploiting this signal, selective prediction (answering only when $k\!=\!3$ runs agree) achieves 87--88\% accuracy at 54--62\% coverage, a 6--14pp gain over single-run baselines, and matches a split-conformal baseline without a held-out calibration set.
A cross-benchmark validation on SWE-bench (50 tasks, 1{,}000 runs) preserves the consistency hierarchy while revealing an $\sim$8$\times$ spread in mean trajectory length across models, and bootstrap analysis shows single-run evaluations misrank models 29.3\% of the time.
\end{abstract}

\section{Introduction}
\label{sec:intro}

Large language model (LLM) based agents that use tools and multi-step reasoning are increasingly deployed for complex tasks \citep{yao2022react, schick2023toolformer}. As these systems move from prototypes to production, understanding their reliability becomes critical.

A fundamental but understudied question is: \emph{how consistent are LLM agents in their behavior?} Given identical inputs, will an agent follow the same reasoning path and arrive at the same answer? This matters because inconsistency complicates debugging, may signal uncertainty exploitable for error detection, and could inform architecture decisions.

We connect behavioral consistency to the statistical frameworks of uncertainty estimation and selective prediction. Each independent agent run can be viewed as an implicit ensemble member \citep{lakshminarayanan2017simple}, and disagreement across runs constitutes a model-agnostic uncertainty signal. Unlike verbalized confidence \citep{xiong2024llms} or learned calibration \citep{kadavath2022language}, this signal requires no model modification, no access to internal states, and no training; it emerges naturally from repeated execution. We show this signal is actionable: it enables selective prediction (abstaining on uncertain inputs), runtime failure detection, and more reliable benchmark evaluation.

We present a systematic empirical study of behavioral consistency in ReAct-style agents \citep{yao2022react}, instantiating three statistical frameworks central to the agentic uncertainty agenda (selective classification, distribution-free calibration, and sequential monitoring) through a single training-free, black-box signal: cross-run agreement. Our key contributions:

\begin{enumerate}
    \item \textbf{Filter $>$ aggregator asymmetry (new for agents).} In multi-step agents, majority voting yields only $+$0--2pp (vs.\ $+$5--17pp for single-turn CoT \citep{wang2022self}) because errors are systematic; \emph{filtering} on disagreement instead yields $+$6--14pp for 3 of 4 models. We explain the gap via a 2$\times$2 task taxonomy in which \emph{consistent-wrong} tasks (5.5--10\%) set the hard ceiling on filtering (Section~\ref{sec:intervention}).

    \item \textbf{Selective prediction with distribution-free validation.} We instantiate the Geifman--El-Yaniv selective-classification framework for agents using cross-run agreement as the selection function, prove a Hoeffding-style concentration bound on the agreement score (Prop.~\ref{prop:concentration}), and show empirically that the resulting risk-coverage curve matches a split-conformal baseline \citep{angelopoulos2023conformal} \emph{without} a held-out calibration set, achieving 87--88\% accuracy at 54--62\% coverage (Sections~\ref{sec:framework}, \ref{sec:conformal}).

    \item \textbf{Behavioral $>$ verbalized confidence.} On a head-to-head comparison on a common 100-question split, behavioral agreement outperforms verbalized self-confidence as a failure detector across three models (Claude, GPT-5, Llama; AUROC 0.65--0.74 vs.\ 0.48--0.55), combining the two does not improve over behavioral alone, and behavioral consistency requires no prompt engineering, no log-probs, and no model modification (Section~\ref{sec:failure_detection}).

    \item \textbf{Cross-benchmark stability.} The consistency hierarchy (Claude $<$ GPT-5 $<$ Gemini $<$ Llama by CV) is preserved exactly across HotpotQA (3-tool QA, 200 tasks) and SWE-bench Verified (unrestricted bash, 50 tasks across 5 repos, 1{,}000 runs), suggesting consistency is a stable, benchmark-independent model property (Section~\ref{sec:swebench}).

    \item \textbf{Ranking instability and divergence mechanism.} Single-run evaluations misrank models 29.3\% [28.4--30.1\%] of the time, undermining standard benchmark methodology (Section~\ref{sec:rankings}); divergence concentrates at step~2 (50.5\% for Llama vs.\ 4--6\% for Claude/GPT-5), and step-level agreement constitutes a running test statistic for sequential monitoring \citep{ramdas2023game} (Section~\ref{sec:divergence}).
\end{enumerate}

\section{Related Work}

\paragraph{LLM-Based Agents.}
ReAct \citep{yao2022react} introduced interleaving reasoning with actions. Subsequent work extended this to web browsing \citep{zhou2023webarena}, code execution \citep{swebench}, and multi-agent systems \citep{wu2023autogen}. We study reliability under repeated execution.

\paragraph{Agent Consistency and Reliability.}
$\tau$-bench \citep{yao2024tau} showed that GPT-4o's pass rate drops from 60\% (pass$^1$) to 25\% (pass$^8$). While $\tau$-bench quantifies \emph{whether} agents are inconsistent, we investigate \emph{where} and \emph{why}: tracing divergence to step~2, correlating consistency with correctness, and identifying path length as a signal. \citet{stroebl2024agents} argue that benchmarks overestimate agent capabilities; our findings support this. Concurrent work has begun formalizing the reliability landscape for agents. ReliabilityBench \citep{gupta2026reliabilitybench} defines a unified reliability surface across consistency, robustness, and fault tolerance; \citet{rabanser2026reliability} propose twelve metrics decomposing reliability along four axes. We provide the deepest empirical analysis of the consistency dimension specifically, demonstrating that it functions not just as a reliability metric but as an actionable uncertainty signal for selective prediction and failure detection.

\paragraph{Agent Uncertainty Quantification.}
\citet{oh2026agentUQ} present the first general formulation of agent UQ and identify the lack of fine-grained multi-run benchmarks as a key challenge; our 8,000-run dataset and consistency analysis contribute directly to this agenda. \citet{zhang2026agentic} propose a training-based dual-process framework for propagating verbalized uncertainty through agent trajectories; behavioral consistency provides a complementary training-free, black-box signal that requires no model modification.

\paragraph{Uncertainty and Selective Prediction.}
Our work connects to three threads. First, \emph{ensemble disagreement}: deep ensembles use prediction variance as an uncertainty estimate \citep{lakshminarayanan2017simple}; behavioral consistency across agent runs is an analogous signal, where each run acts as an implicit ensemble member. Recent work has applied conformal prediction to LLM outputs \citep{kumar2023conformal, quach2024conformal} and to robot planning under uncertainty \citep{ren2023robots}, providing distribution-free coverage guarantees. \citet{farquhar2024detecting} extend semantic uncertainty to detect LLM hallucinations via predictive entropy, requiring access to token-level log-probabilities. Behavioral consistency provides a complementary signal at the trajectory level: it captures uncertainty in the agent's \emph{plan}, not just its output tokens, and operates as a black-box method requiring only repeated execution. Crucially, two runs can produce the same final answer via divergent tool sequences, a source of plan-level uncertainty that token-level entropy cannot detect. Second, \emph{self-consistency}: \citet{wang2022self} show that majority voting over sampled reasoning chains improves accuracy; we extend this from single-turn CoT to multi-step agentic settings and additionally study consistency as a \emph{diagnostic} signal, not just an ensembling strategy (Section~\ref{sec:intervention}). Third, \emph{selective prediction}: \citet{geifman2017selective} and \citet{kamath2020selective} show that abstaining on uncertain inputs improves precision; our consistency-filtered results instantiate this for agents. Unlike verbalized confidence \citep{tian2023just, xiong2024llms}, behavioral consistency requires no model modification; it emerges naturally from repeated execution. Prior work on LLM calibration \citep{kadavath2022language} and semantic uncertainty \citep{kuhn2023semantic} focuses on single-turn outputs; we extend to multi-step agentic behavior.

\section{Methodology}

\subsection{Agent Architecture}

We implement a ReAct-style agent with three tools: \textbf{Search(query)} returns document titles via keyword matching, \textbf{Retrieve(title)} returns full text, and \textbf{Finish(answer)} terminates with a final answer. The agent follows the standard think-act-observe loop.

\subsection{Experimental Setup}

\paragraph{Dataset.} We use 200 questions from HotpotQA \citep{yang2018hotpotqa} validation (distractor setting), requiring multi-hop reasoning over 10 paragraphs (2 gold, 8 distractors).

\paragraph{Models.} Four models across four providers: Llama 3.1 70B (Meta), GPT-5 (OpenAI), Claude Sonnet 4.5 (Anthropic), and Gemini 3 Pro (Google). The per-run figures (Sections~\ref{sec:divergence}, \ref{sec:conformal}, Figures~\ref{fig:calibration}--\ref{fig:step_agreement}) and the verbalized-confidence comparison (Section~\ref{sec:failure_detection}, Table~\ref{tab:verb_conf}) use 100-question splits with full per-run trajectories ($k\!=\!10$ for the figures, $k\!=\!5$ for Table~\ref{tab:verb_conf}). As an earlier-generation reference, we additionally evaluated GPT-4o on a separate 100-question subset (Appendix~\ref{app:gpt4o}).

\paragraph{Runs.} For each question--model pair, 10 independent runs at temperature 0.7, yielding 8,000 total runs. Gemini 3 Pro required an additional \texttt{top\_p\,=\,0.95} parameter to achieve stochastic sampling (Appendix~\ref{sec:gemini_note}). Temperature ablation: 20 questions $\times$ 5 runs $\times$ 4 models $\times$ 3 temperatures.

\paragraph{Statistical note.} All tests are two-sided; raw $p$-values without multiple-comparison correction given the exploratory nature.

\subsection{Metrics}

\textbf{Answer consistency}: fraction of runs producing the most common answer. \textbf{Action sequence diversity}: number of unique action sequences across $N$ runs. \textbf{Step variance ratio}: $(\max - \min)/\text{mean}$ of step counts. \textbf{First divergence point}: earliest step where runs take different actions. \textbf{Correctness}: fuzzy string matching (answer contains gold or vice versa, case-insensitive); results under exact match and token F1 in Appendix~\ref{app:metrics}.

\subsection{Behavioral Consistency as Selective Prediction}
\label{sec:framework}

We formalize behavioral consistency within the selective classification
framework \citep{geifman2017selective}. Let $f$ be an agent, $x$ a task, and
$\{f_1(x), \ldots, f_k(x)\}$ be $k$ independent runs at fixed temperature.
Define the \emph{agreement score}:
\begin{equation}
  s(x) = \max_a \frac{1}{k} \sum_{i=1}^k \mathbf{1}[f_i(x) = a]
  \label{eq:agreement}
\end{equation}
i.e., the fraction of runs producing the plurality answer. The
\emph{selection function} $g_\tau(x) = \mathbf{1}[s(x) \geq \tau]$ admits
task $x$ when agreement exceeds threshold $\tau$. The \emph{selective risk}
at threshold $\tau$ is:
\begin{equation}
  R(\tau) = \frac{\sum_{i} \ell(f(x_i), y_i) \cdot g_\tau(x_i)}
                 {\sum_{i} g_\tau(x_i)}
  \label{eq:selective_risk}
\end{equation}
where $\ell$ is 0-1 loss and $f(x_i)$ is the plurality answer.

\begin{proposition}[Concentration of agreement score]
\label{prop:concentration}
Let $s^*(x) = \max_a \Pr[f(x) = a]$ be the true plurality probability
for task $x$, and $\hat{s}_k(x)$ the empirical agreement score from $k$
i.i.d.\ runs (Eq.~\ref{eq:agreement}). Then for any $\delta > 0$:
\begin{equation}
  \Pr\!\Big[|\hat{s}_k(x) - s^*(x)| \leq
    \sqrt{\tfrac{\ln(2/\delta)}{2k}}\,\Big] \geq 1 - \delta
  \label{eq:hoeffding}
\end{equation}
If $\hat{s}_k(x) \geq \tau$ and
$\epsilon_k \!=\! \sqrt{\ln(2/\delta)/(2k)}$, the true plurality probability
satisfies $s^*(x) \geq \tau - \epsilon_k$ with probability $\geq 1-\delta$.
\end{proposition}

\noindent\emph{Proof.}\; Hoeffding's inequality applied to $k$ i.i.d.\
Bernoulli indicators $\mathbf{1}[f_i(x) = a^*]$. \hfill$\square$

\vspace{0.5em}\noindent
For $k\!=\!3$, $\epsilon_3 = 0.78$ at $\delta\!=\!0.05$ (vacuous); for
$k\!=\!10$, $\epsilon_{10} = 0.43$; the bound becomes informative at
$k \geq 20$ ($\epsilon_{20} = 0.30$). Despite the looseness at small~$k$, our empirical
risk-coverage curves (Figure~\ref{fig:risk_coverage}) confirm that
$R(\tau)$ decreases monotonically with $\tau$ at $k\!=\!10$, and
Section~\ref{sec:conformal} validates the agreement score against a
split-conformal baseline.

\paragraph{Exchangeability assumption.}
Eq.~\ref{eq:hoeffding} requires $\{f_i(x)\}_{i=1}^k$ to be i.i.d.\
given $x$. Our setup satisfies this: each run uses an independent API
call with a fresh random seed and stateless tools (search and retrieve
operate over a fixed corpus).
Exchangeability would be violated by stateful tools (e.g., a shared
database modified by earlier runs) or autoregressive dependence across
runs. While we cannot rule out provider-side prompt caching, such
caching affects only latency, not output distributions, under standard
API contracts. Our design avoids all substantive violations.

\section{Results}

\begin{figure*}[t]
\centering
\begin{subfigure}[t]{0.48\textwidth}
\centering
\includegraphics[width=\textwidth]{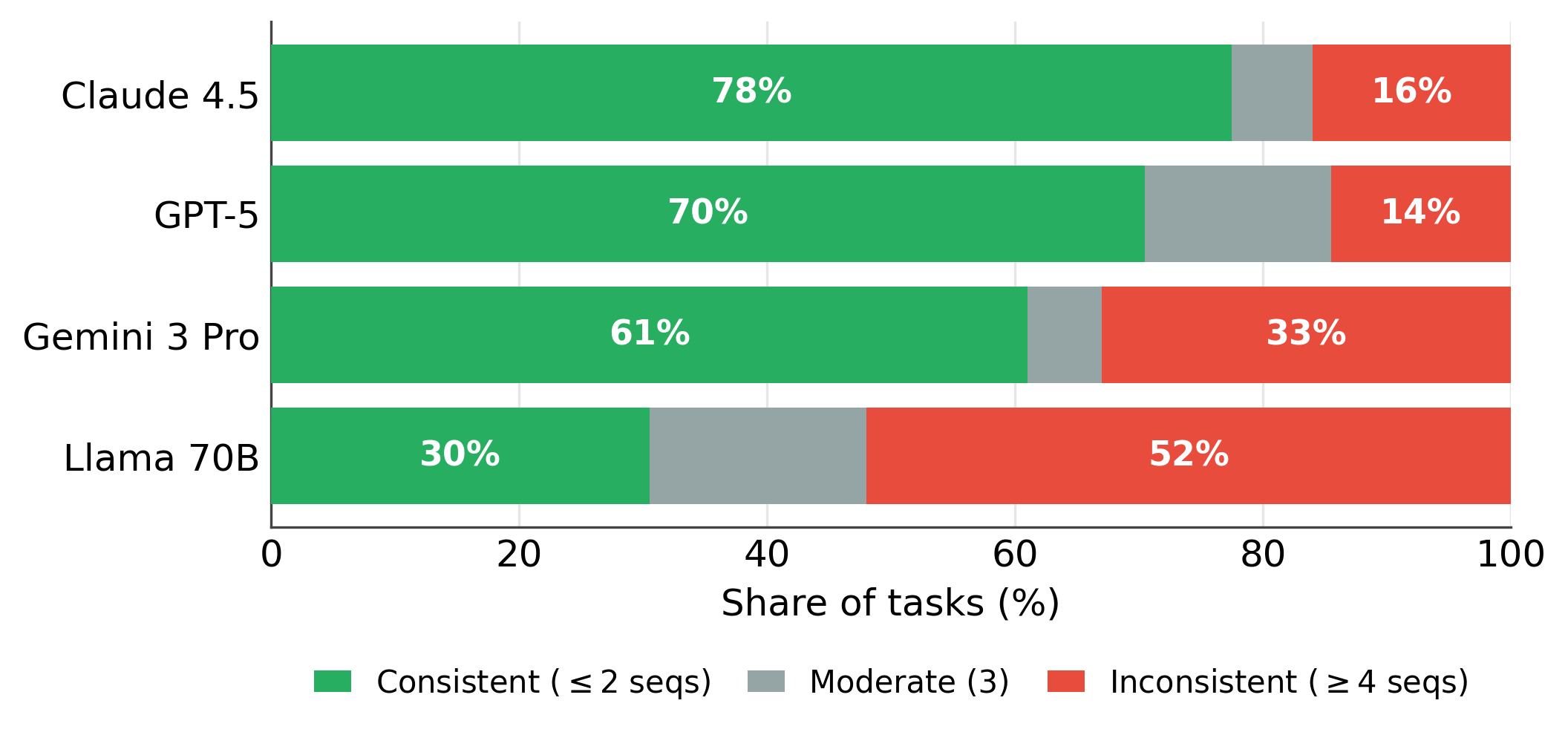}
\caption{Share of tasks by behavioral consistency level (unique action sequences).}
\label{fig:histogram}
\end{subfigure}
\hfill
\begin{subfigure}[t]{0.48\textwidth}
\centering
\includegraphics[width=\textwidth]{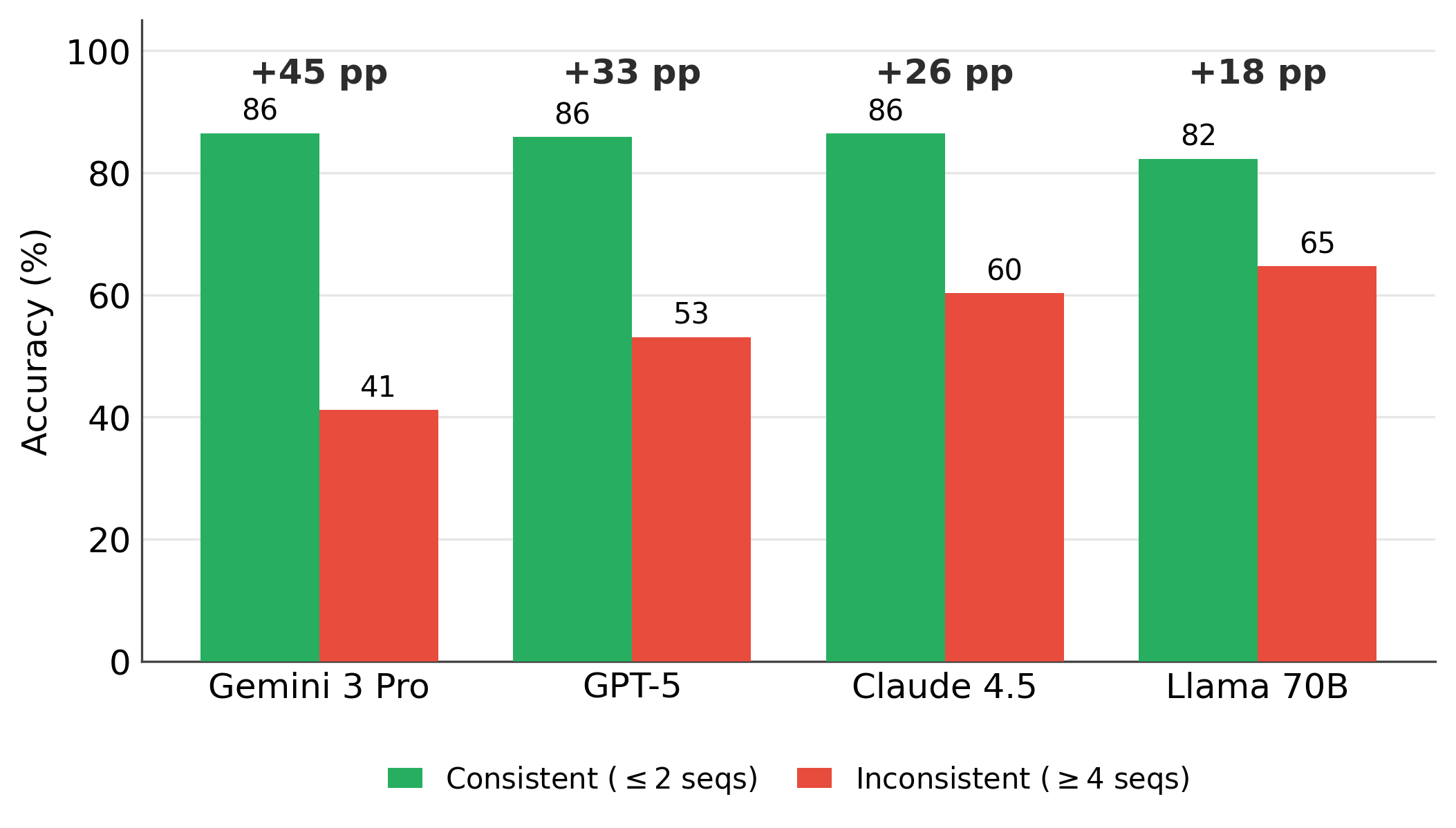}
\caption{Correctness for consistent ($\leq$2 seqs) vs.\ inconsistent ($\geq$4 seqs) tasks.}
\label{fig:correctness}
\end{subfigure}
\caption{Behavioral consistency varies across models (left) and is associated with correctness (right).}
\label{fig:main}
\end{figure*}

\subsection{Overall Model Comparison}

Table~\ref{tab:overall} summarizes 8,000 runs. Claude Sonnet 4.5 achieves both the highest accuracy and consistency (2.3 unique sequences). Llama 3.1 70B shows the most behavioral variance (4.2 unique sequences), followed by Gemini 3 Pro (3.2). All four models exhibit measurable behavioral inconsistency even on identical inputs.

\begin{table}[t]
\centering
\small
\caption{Performance across four models (200 tasks each). Seqs = unique action sequences; Var = step variance ratio.}
\label{tab:overall}
\begin{tabular}{lcccc}
\toprule
\textbf{Model} & \textbf{Correct} & \textbf{Seqs} & \textbf{Steps} & \textbf{Var} \\
\midrule
Claude Sonnet 4.5 & \textbf{81.5\%} & \textbf{2.3} & 5.1 & \textbf{22.2\%} \\
GPT-5 & 79.6\% & 2.4 & 4.2 & 36.0\% \\
Llama 3.1 70B & 73.7\% & 4.2 & 4.7 & 64.9\% \\
Gemini 3 Pro & 72.2\% & 3.2 & 5.4 & 38.5\% \\
\bottomrule
\end{tabular}
\end{table}

\subsection{Consistency Is Associated with Correctness}

Table~\ref{tab:consistency} shows the central finding: consistent tasks ($\leq$2 unique sequences) achieve 82--87\% accuracy while inconsistent tasks ($\geq$4 sequences) achieve 41--65\%. The gap is significant for all four models (Mann-Whitney $U$, all $p<0.001$; rank-biserial $r=-0.35$ to $-0.61$).\footnote{All four consistency-gap tests remain significant at $\alpha=0.05$ after Holm-Bonferroni correction (most stringent adjusted $p = 3.44 \times 10^{-5}$).}

\begin{table}[t]
\centering
\small
\caption{Correctness (\%) by consistency level. Gap in pp. $r$ = rank-biserial; $n$ = consistent/inconsistent counts. All $p<0.001$.}
\label{tab:consistency}
\begin{tabular}{l@{\hskip 4pt}c@{\hskip 4pt}c@{\hskip 4pt}c@{\hskip 4pt}c@{\hskip 4pt}c}
\toprule
\textbf{Model} & \textbf{Cons.} & \textbf{Incons.} & \textbf{Gap} & \textbf{$r$} & \textbf{$n$} \\
 & ($\leq$2) & ($\geq$4) & & & \\
\midrule
Gemini 3 Pro & 86.5 & 41.2 & 45.3 & $-$.61 & 122/66 \\
GPT-5 & 85.9 & 53.1 & 32.8 & $-$.43 & 141/29 \\
Claude Son.\ 4.5 & 86.4 & 60.3 & 26.1 & $-$.39 & 155/32 \\
Llama 70B & 82.3 & 64.7 & 17.6 & $-$.35 & 61/104 \\
\bottomrule
\end{tabular}
\end{table}

\paragraph{Controlling for difficulty.} Task difficulty confounds the consistency--correctness relationship. We compute a difficulty proxy (mean correctness across all models) and bin questions into Easy ($\geq$0.80, $n=132$), Medium (0.40--0.79, $n=32$), and Hard ($<$0.40, $n=36$). Within Easy tasks, the consistency gap is positive for Gemini ($+$16.5pp), Llama ($+$11.2pp), and GPT-5 ($+$6.8pp) but near zero for Claude ($-$0.5pp), where ceiling effects leave little room for a gap (both consistent and inconsistent Easy tasks exceed 90\% accuracy). Within Medium tasks, the gap is positive for all four models ($+$7 to $+$34pp), though cell sizes are small ($n_c$=6--16, $n_i$=10--22; Appendix~\ref{app:difficulty}). For Hard tasks, the gap reverses for three of four models under fuzzy match. This reversal is a \emph{consistent-wrong} phenomenon: of 61 consistent Hard task-model pairs, 52 (85.2\%) are scored incorrect under fuzzy match. Approximately 30\% of these are false negatives, surface-form mismatches where the model produced a semantically correct answer (e.g., ``Pasek and Paul'' vs.\ gold ``Pasek \& Paul''; ``Kelly Osbourne'' vs.\ ``Kelly Lee Osbourne''). Annotation protocol: one annotator classified each case as surface-form mismatch vs.\ genuinely incorrect; all annotated cases and classification criteria will be released. When we recompute Hard stratum gaps under token F1 $\geq$ 0.5, all four models show non-negative gaps ($+$9 to $+$16pp; Table~\ref{tab:f1_hard}), confirming that the reversal is a fuzzy-match artifact. The remaining $\sim$70\% represent true consistent-wrong behavior, where the agent confidently commits to a wrong answer. Full analysis in Appendix~\ref{app:difficulty}. The within-stratum association is weaker than the aggregate effect, consistent with difficulty being the primary confound. Partial correlations (controlling for difficulty as a continuous variable) remain significant for three of four models: Gemini ($r=-0.49$, $p<0.001$), GPT-5 ($r=-0.16$, $p=0.025$), and Llama ($r=-0.16$, $p=0.028$). The non-significant partial correlation for Claude ($r=0.01$, $p=0.88$) reflects a ceiling effect: Claude is consistent-correct on 68\% of tasks, leaving insufficient within-model variance for the partial correlation to detect.

\subsection{Consistency as a Runtime Failure Detector}
\label{sec:failure_detection}

We frame failure detection as binary classification: predicting whether the majority answer is incorrect using only consistency metrics. Table~\ref{tab:detection} reports the best feature per model.

\begin{table}[t]
\centering
\small
\caption{Failure detection via consistency metrics. H = answer entropy. Baseline precision equals the per-model error rate.}
\label{tab:detection}
\begin{tabular}{llccc}
\toprule
\textbf{Model} & \textbf{Feat.} & \textbf{Prec.} & \textbf{Rec.} & \textbf{AUROC} \\
\midrule
Gemini 3 Pro & Seq$>$4 & 72.1\% & 54.4\% & .775 \\
Llama 70B & H$>$1.5 & 44.9\% & 50.0\% & .714 \\
GPT-5 & H$>$1.5 & 35.5\% & 56.4\% & .688 \\
Claude Son.\ 4.5 & H$>$0.5 & 23.9\% & 63.6\% & .619 \\
\bottomrule
\end{tabular}
\end{table}

For the most variable models (Llama, GPT-5, Gemini), consistency metrics achieve AUROC 0.69--0.78, substantially above the 0.50 random baseline. Even for the most consistent model (Claude), AUROC reaches 0.62, validating consistency as a lightweight runtime failure signal across all four models. Gemini's best predictor is unique sequences rather than answer entropy ($>$4 sequences indicates likely failure); with 66 inconsistent tasks in 200 questions, Gemini provides sufficient signal for the sequence-count feature to outperform entropy-based thresholds.

\paragraph{Comparison to verbalized confidence.}
Table~\ref{tab:verb_conf} compares behavioral consistency to verbalized confidence (post-hoc 0--100 self-rating elicited via a single additional API call after each run) as failure detectors.\footnote{Behavioral AUROCs in Table~\ref{tab:verb_conf} differ from Table~\ref{tab:detection} because they are computed on a 100-task subset with $k\!=\!5$ runs rather than the full 200 tasks with $k\!=\!10$.} For the three models on a common 100-question split (Claude, GPT-5, Llama), behavioral consistency outperforms verbalized confidence (Claude 0.651 vs.\ 0.522; GPT-5 0.734 vs.\ 0.547; Llama 0.739 vs.\ 0.481; Figure~\ref{fig:verbconf}). Spearman correlation between the two signals is low (largely orthogonal), and combining them via logistic regression (5-fold CV) does not improve over behavioral consistency alone, suggesting behavioral consistency subsumes the predictive information in verbalized confidence. These results are consistent with prior findings that behavioral consistency is a competitive, training-free failure signal requiring no prompt engineering, no log-probs, and no model modification.

\begin{table}[t]
\centering
\small
\caption{Failure detection AUROC: behavioral consistency vs.\ verbalized confidence vs.\ combined (logistic regression, 5-fold CV), on a common 100-question split ($\times$ 5 runs). Behavioral consistency outperforms verbalized confidence for all three models, and combining the two does not improve over behavioral consistency alone.}
\label{tab:verb_conf}
\begin{tabular}{lccc}
\toprule
\textbf{Model} & \textbf{Behavioral} & \textbf{Verbalized} & \textbf{Combined} \\
\midrule
Claude Son.\ 4.5 & .651 & .522 & .646 \\
GPT-5            & .734 & .547 & .727 \\
Llama 70B        & .739 & .481 & .701 \\
\bottomrule
\end{tabular}
\end{table}

\begin{figure}[t]
\centering
\includegraphics[width=\columnwidth]{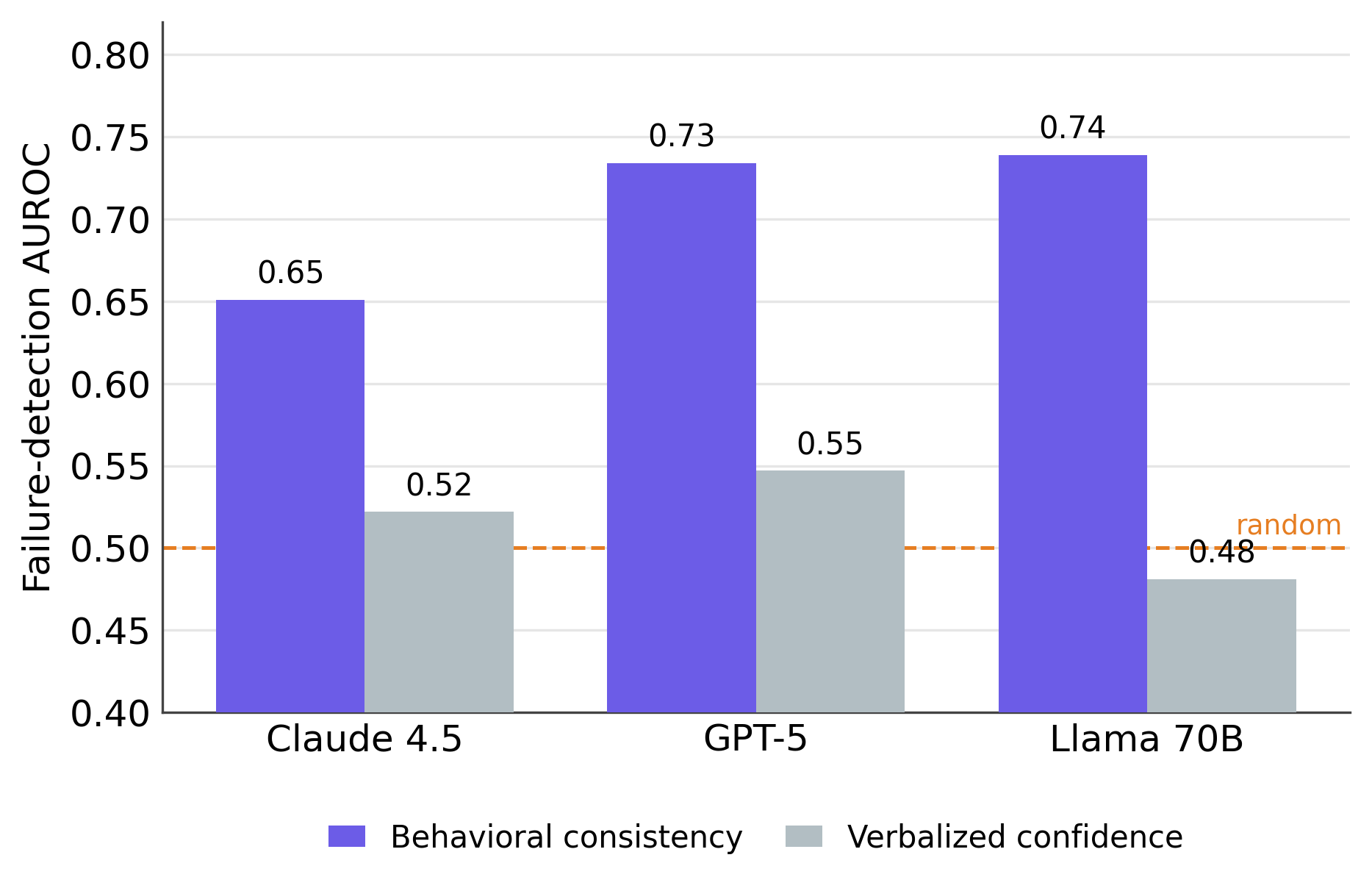}
\caption{Behavioral consistency outperforms verbalized self-confidence as a failure detector for all three models, with verbalized confidence near or below the random baseline.}
\label{fig:verbconf}
\end{figure}

\subsection{Majority Voting and Selective Prediction}
\label{sec:intervention}

Can consistency be \emph{exploited}, not just observed? We simulate a majority-vote intervention using existing multi-run data: for budget $k$, subsample $k$ runs per question, take the majority answer, and evaluate accuracy (500 bootstrap iterations).

\begin{table}[t]
\centering
\small
\caption{Majority-vote accuracy (\%) and selective prediction. $k\!=\!1$ is single-run baseline. Sel.\ = unanimous $k\!=\!3$ (answer only when all 3 agree).}
\label{tab:vote}
\begin{tabular}{l@{\hskip 5pt}c@{\hskip 5pt}c@{\hskip 5pt}c@{\hskip 5pt}c@{\hskip 5pt}c}
\toprule
\textbf{Model} & $k\!=\!1$ & $k\!=\!3$ & \textbf{Gain} & \textbf{Sel.\ Acc} & \textbf{Cov.} \\
\midrule
Llama 70B & 73.6 & 75.6 & +2.0 & 87.2 & 53.8 \\
Claude Son.\ 4.5 & 81.5 & 82.2 & +0.7 & 87.8 & 61.8 \\
GPT-5 & 79.6 & 80.1 & +0.5 & 88.1 & 55.2 \\
Gemini 3 Pro & 72.2 & 72.3 & $+$0.1 & 73.4 & 64.7 \\
\bottomrule
\end{tabular}
\end{table}

We instantiate the selective classification framework \citep{geifman2017selective} for multi-step agents, using behavioral agreement as the selection function $g_\tau$ (Eq.~\ref{eq:agreement}). Table~\ref{tab:vote} reveals that consistency works far better as a \emph{filter} than an \emph{aggregator}, a striking contrast with single-turn self-consistency \citep{wang2022self}, where majority voting improves CoT accuracy by 5--17pp. In our multi-step agentic setting, voting yields only +0--2pp because agentic errors are \emph{systematic}: an early wrong action commits the agent to an incorrect trajectory, producing consistent-wrong answers that voting cannot correct. \emph{Selective prediction} (answering only when all $k\!=\!3$ runs agree) is the stronger intervention: for Llama, Claude, and GPT-5, unanimous agreement achieves 87--88\% accuracy (6--14pp above baselines) at 54--62\% coverage. Figure~\ref{fig:filter_aggregator} visualizes this asymmetry directly.

\begin{figure*}[t]
\centering
\includegraphics[width=0.72\textwidth]{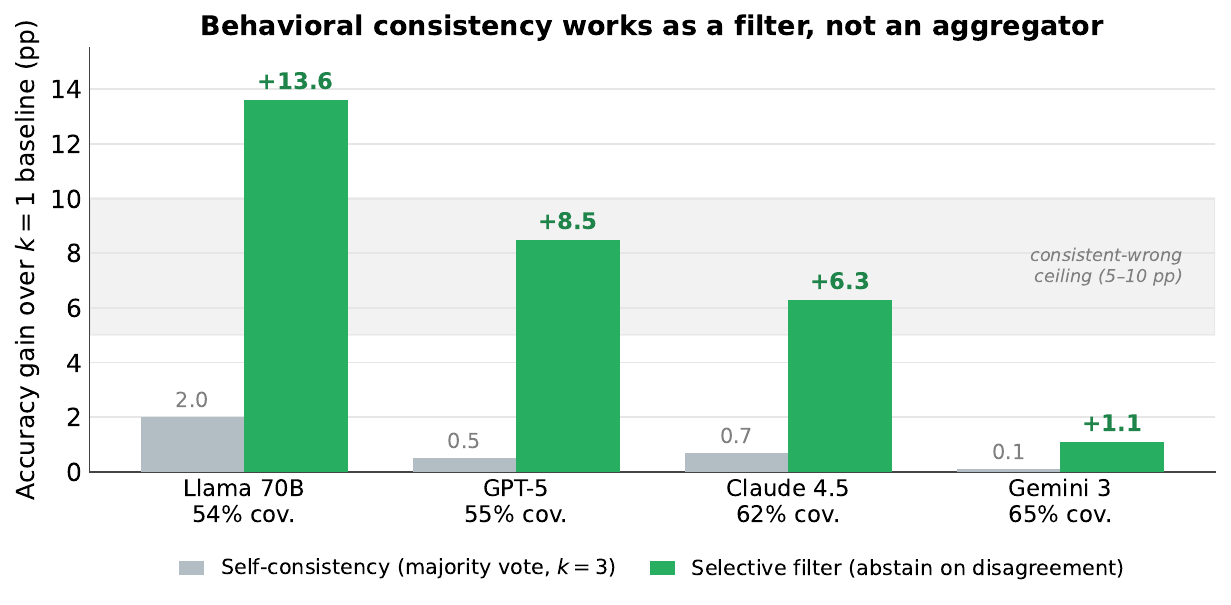}
\caption{Filter $>$ aggregator. Self-consistency (majority vote, $k\!=\!3$) gains only 0--2pp, whereas behavioral filtering (answer only when all $k\!=\!3$ runs agree) gains 6--14pp for three of four models at 54--62\% coverage. The shaded band marks the consistent-wrong ceiling (5.5--10\% of tasks), where the agent confidently commits to an incorrect answer, the fundamental limit on filtering gains.}
\label{fig:filter_aggregator}
\end{figure*}
Table~\ref{tab:taxonomy_main} gives the 2$\times$2 task-outcome taxonomy: consistent-correct tasks (25--68\%) are retained; inconsistent-wrong tasks (6--19\%) are correctly excluded; \emph{consistent-wrong} tasks (5.5--10\%) set the hard ceiling on filtering gains.

\begin{table}[t]
\centering
\small
\caption{Task-outcome taxonomy (all 200 tasks; per-model ranges). Filtering retains the top-left cell and excludes the bottom-right; Consistent-Wrong tasks (top-right) set the performance ceiling.}
\label{tab:taxonomy_main}
\begin{tabular}{lcc}
\toprule
 & \textbf{Majority Correct} & \textbf{Majority Wrong} \\
\midrule
\textbf{Consistent} ($\leq$2 seqs) & Cons.-Correct & \textbf{Cons.-Wrong} \\
 & 25--68\% & 5.5--10\% \\
\textbf{Inconsistent} ($\geq$4 seqs) & Incons.-Lucky & Incons.-Wrong \\
 & 8--38\% & 6--19\% \\
\bottomrule
\end{tabular}
\end{table}

Notably, 61\% of Claude's errors are consistent-wrong vs.\ only 25\% of Llama's, explaining why Llama benefits most from voting. Stricter thresholds yield further gains: unanimous 5/5 agreement achieves 90.2\% accuracy for GPT-5 at 48\% coverage (full taxonomy and thresholds in Appendix~\ref{app:vote}).

The computational cost of $k\!=\!3$ runs is $3\times$ the single-run cost. At current API pricing, $k\!=\!3$ HotpotQA runs cost approximately \$0.03--0.15 per task depending on the model. The accuracy gain (6--14pp) and the ability to abstain on uncertain inputs justify this overhead in deployments where reliability matters more than throughput.

Figure~\ref{fig:risk_coverage} plots the risk-coverage tradeoff as the agreement threshold $\tau$ varies. The monotonically decreasing curves confirm that behavioral agreement is a valid selection function: stricter agreement requirements yield lower risk at the cost of reduced coverage, consistent with Eq.~\ref{eq:selective_risk}.

\begin{figure}[t]
\centering
\includegraphics[width=\columnwidth]{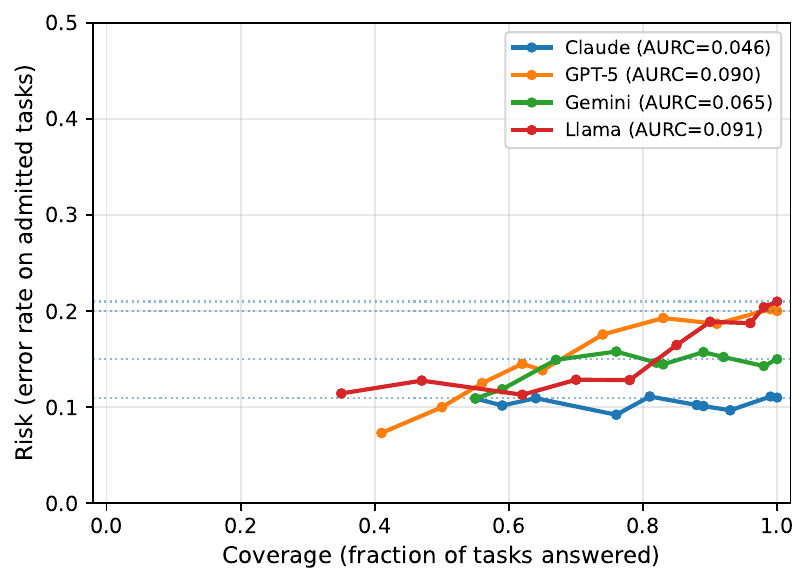}
\caption{Risk-coverage curve. As the agreement threshold $\tau$ increases (moving left), admitted tasks have lower error rates. The monotonic decrease confirms that behavioral agreement is a valid selection function in the Geifman--El-Yaniv framework \citep{geifman2017selective}. AURC \citep{zhou2025aurc} values per model in legend.}
\label{fig:risk_coverage}
\end{figure}

Figure~\ref{fig:calibration} shows the relationship between answer consistency and actual correctness. Behavioral consistency is an imperfect probability estimate: agreement levels do not match correctness exactly (ECE = 0.148, 0.153, 0.172, 0.090 for Claude, GPT-5, Gemini, Llama respectively). This calibration gap, analogous to miscalibration in neural network confidence \citep{guo2017calibration}, suggests that multi-run consistency is a useful but imperfect uncertainty signal. Importantly, this calibration property emerges without any training or model modification.

\begin{figure}[t]
\centering
\includegraphics[width=\columnwidth]{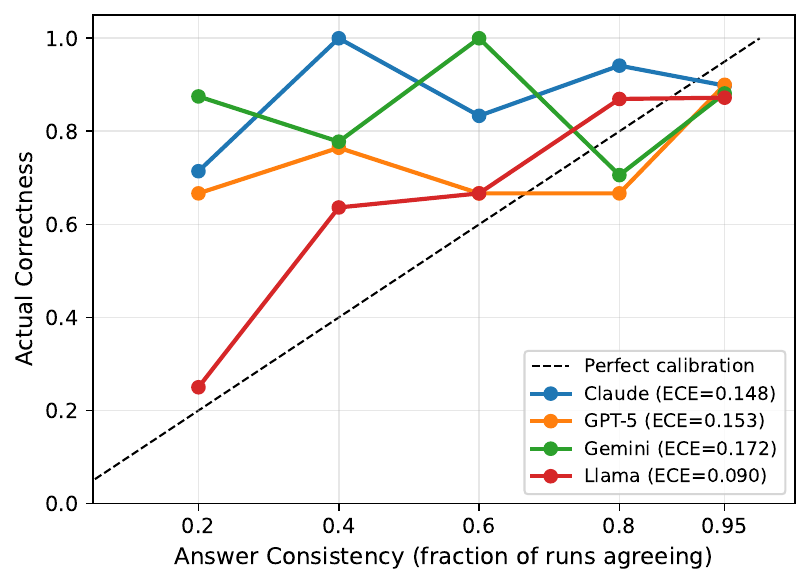}
\caption{Behavioral calibration curve. Answer consistency (fraction of runs agreeing) vs.\ actual correctness. Dashed diagonal = perfect calibration. ECE values per model in legend.}
\label{fig:calibration}
\end{figure}

\subsection{Divergence Occurs Early; Path Length Matters}
\label{sec:divergence}

Pooled across four models, 18.8\% [95\% CI: 16.2--21.6\%] of tasks diverge by step~2, but this rate varies dramatically: Llama diverges early in 50.5\% [43.6--57.4\%] of tasks vs.\ 4--6\% for Claude and GPT-5. Among Llama's early-diverging tasks, accuracy is 71.7\% vs.\ 85.8\% for late-diverging tasks. The first search query largely determines the trajectory.

Path length is strongly associated with behavioral consistency: Spearman $\rho$ between mean steps per task and unique action sequences is 0.74 (Llama), 0.54 (Gemini), 0.53 (Claude), 0.50 (GPT-5); pooled $\rho=0.51$, all $p<0.001$. Path length also tracks correctness in the early-vs-late divergence comparison above (e.g., Llama 71.7\% accuracy on early-diverging tasks vs.\ 85.8\% on late-diverging tasks). Longer paths indicate backtracking and uncertainty; each additional step is an opportunity to diverge and err.

Figure~\ref{fig:step_agreement} quantifies this timing by measuring the mean pairwise action-sequence agreement rate at each step across 100 tasks $\times$ 10 runs for all four models. Agreement drops after step~1 for all models; Llama exhibits the steepest decline (0.38 at step~1, 0.27 at step~2, 0.13 at step~3), while Claude retains high agreement through step~2 (0.74), consistent with its lower behavioral variance. Step-2 action agreement predicts final majority-vote correctness with AUROC 0.60 (Claude), 0.58 (Gemini), 0.50 (GPT-5), 0.49 (Llama), near chance, indicating that \emph{where} divergence first occurs is informative about model behavior but that \emph{answer-level} consistency (Table~\ref{tab:detection}, AUROC 0.62--0.78) is the stronger failure signal.

\begin{figure}[t]
\centering
\includegraphics[width=\columnwidth]{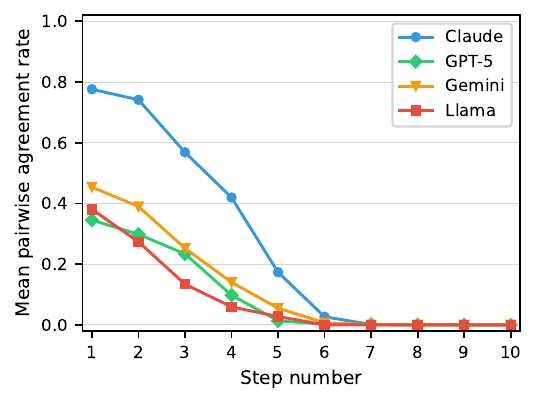}
\caption{Mean pairwise action-sequence agreement rate per step (100 tasks $\times$ 10 runs, four models). Llama diverges earliest and most sharply; Claude retains high agreement through step~2. Step-2 agreement alone yields AUROC 0.49--0.60 for failure prediction.}
\label{fig:step_agreement}
\end{figure}

\subsection{Comparison to Split-Conformal Selective Prediction}
\label{sec:conformal}

A natural question is whether the agreement score provides genuine
selective-prediction value, or whether a standard conformal procedure
would achieve comparable results. We compare against a
\emph{split-conformal baseline}: calibrate the agreement threshold on
70\% of tasks and evaluate on the held-out 30\%, repeating over 500
random splits. This provides a distribution-free coverage guarantee
at any target level \citep{angelopoulos2023conformal, bates2021distribution}.

Figure~\ref{fig:k_sensitivity} plots coverage vs.\ accuracy as the
unanimity requirement $k$ increases from 2 to 10 for all four models
(100 tasks $\times$ 10 runs each). Behavioral
consistency closely tracks the conformal baseline for all four
models, confirming that the agreement score is a well-calibrated
selection function \emph{without requiring a separate calibration set}.
This is the key practical advantage: conformal methods need a held-out
calibration split, while behavioral consistency uses only the $k$ runs
already executed for prediction, so no data is ``wasted'' on calibration.

\begin{figure}[t]
\centering
\includegraphics[width=\columnwidth]{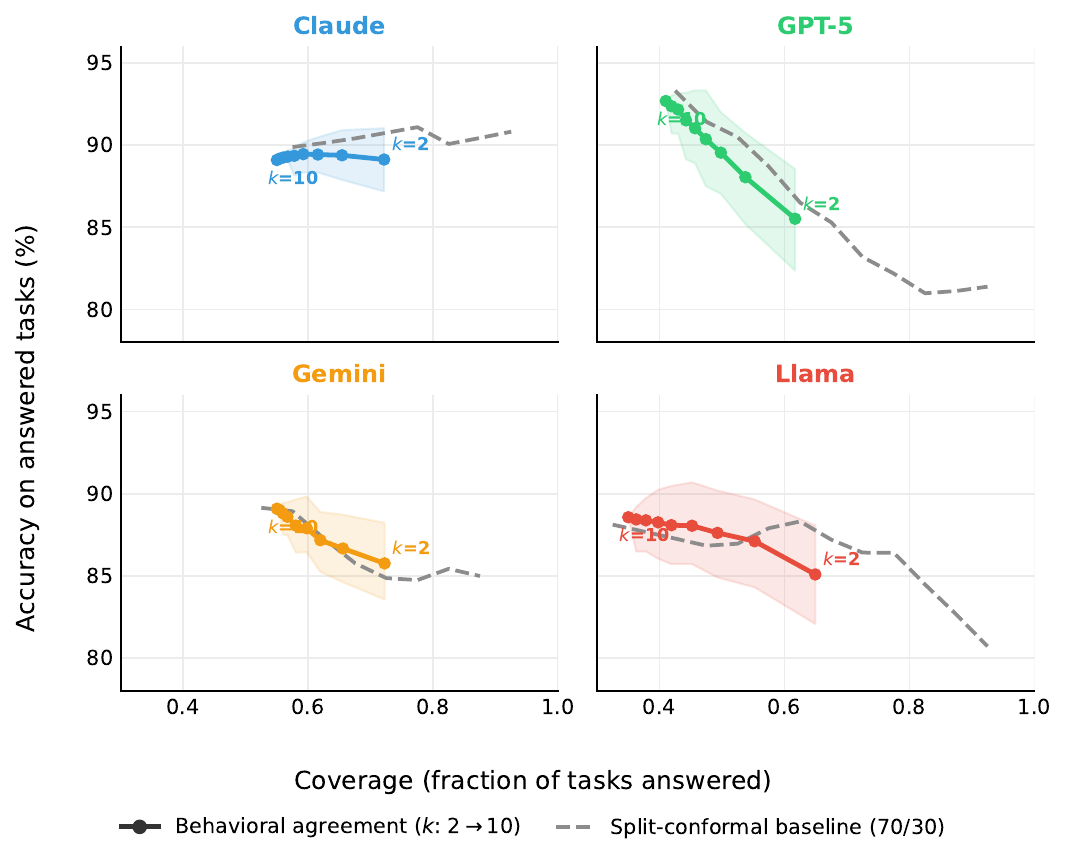}
\caption{Coverage vs.\ accuracy as unanimity requirement $k$ increases
from 2 to 10. Shaded bands = 10--90\% bootstrap intervals. Dashed:
split-conformal baseline (70/30 calibration split). Behavioral
consistency matches the conformal baseline without
requiring a calibration set.}
\label{fig:k_sensitivity}
\end{figure}

\subsection{Cross-Benchmark Validation: SWE-bench}
\label{sec:swebench}

To test generalizability beyond HotpotQA's minimal action space, we conduct a cross-benchmark validation on SWE-bench Verified \citep{swebench, swebench_verified}: 50 tasks across 5 open-source repositories (astropy, django, matplotlib, scikit-learn, sympy), 5 runs per model-task pair (1{,}000 total runs), using all four models with an identical bash-only scaffold (temperature 0.5).

\begin{table}[t]
\centering
\small
\caption{SWE-bench cross-benchmark validation (50 tasks $\times$ 5 runs across 5 repositories). CV = std/mean of step counts. The HotpotQA consistency hierarchy is perfectly preserved across all four models.}
\label{tab:swebench}
\begin{tabular}{lcccc}
\toprule
& \textbf{Claude} & \textbf{GPT-5} & \textbf{Gemini} & \textbf{Llama} \\
\midrule
Tasks & 50 & 50 & 50 & 50 \\
Repositories & 5 & 5 & 5 & 5 \\
CV & \textbf{17.7\%} & 30.3\% & 46.8\% & 68.9\% \\
Mean steps & 53.3 & 10.5 & 81.8 & 27.9 \\
\bottomrule
\end{tabular}
\end{table}

The consistency hierarchy is perfectly preserved across all four models (Claude $<$ GPT-5 $<$ Gemini $<$ Llama by CV), matching the HotpotQA ranking. This hierarchy holds across a 3-tool QA environment (HotpotQA) and unrestricted-bash coding tasks spanning 5 open-source repositories (SWE-bench), suggesting that behavioral consistency is a stable, benchmark-independent model property. Resolution rates on the original 10-task astropy subset were Claude 58\%, GPT-5 32\%, Llama 4\%, preserving the accuracy hierarchy. Full cross-benchmark comparison in Appendix~\ref{app:swebench}.

\paragraph{Trajectory length spans $\sim$8$\times$.} Mean steps range from GPT-5 (10.5) to Gemini (81.8), with Llama (27.9) and Claude (53.3) in between, a 7.8$\times$ ratio the 3-tool HotpotQA setup masks (4.2--5.4; Table~\ref{tab:overall}). Length and CV are \emph{not} co-monotone (Gemini is the longest yet second-most-consistent; Llama is shorter than Claude/Gemini yet most variable), so they capture distinct dimensions, with direct cost-, latency-, and step-cap implications.

\subsection{Ranking Instability}
\label{sec:rankings}

We perform 10,000 bootstrap iterations across four models on all 200 questions, sampling one run per question per model. \textbf{29.3\% [28.4--30.1\%] of single-run evaluations produce a ranking differing from the multi-run ground truth} (Claude 81.5\% $>$ GPT-5 79.6\% $>$ Llama 73.7\% $>$ Gemini 72.2\%; Figure~\ref{fig:ranking}). Nearly one in three single-run evaluations would misrank models. The instability stems from narrow accuracy gaps between adjacent models and within-model variance. Agent benchmarks should report multi-run statistics.

\begin{figure}[t]
\centering
\includegraphics[width=\columnwidth]{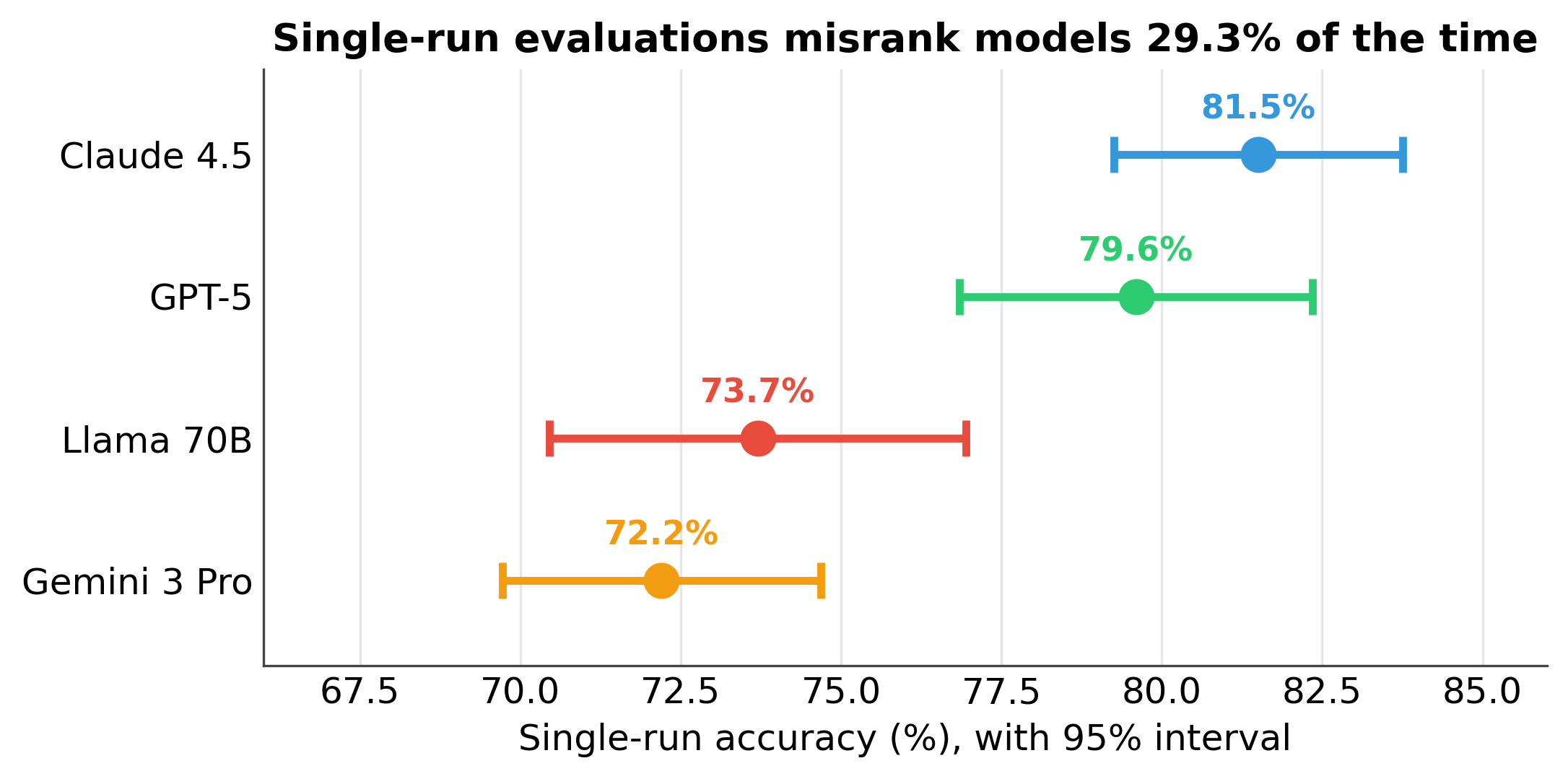}
\caption{Single-run accuracy with 95\% intervals (one run sampled per question). Adjacent models' intervals overlap, with Claude/GPT-5 near the top and Llama/Gemini at the bottom (1.5pp apart), so single-run evaluations reorder models 29.3\% of the time.}
\label{fig:ranking}
\end{figure}

\section{Discussion}

\paragraph{Practical implications.}
Consistency monitoring is feasible and actionable: even $k\!=\!3$ runs deliver the selective-prediction gains of Section~\ref{sec:intervention}, the ceiling is set by the consistent-wrong fraction (5.5--10\%; Table~\ref{tab:taxonomy}), and the step~2 bottleneck (Section~\ref{sec:divergence}) suggests that improving first-query formulation could reduce downstream variance.

\paragraph{Consistency as implicit, distribution-free calibration.}
Traditional calibration asks whether a model's \emph{stated} confidence matches its accuracy \citep{kadavath2022language}; behavioral consistency offers an alternative \emph{behavioral} notion in which cross-run disagreement plays the role of deep-ensemble variance \citep{lakshminarayanan2017simple} and, crucially, requires no calibration set, no log-probs, and no prompting strategy. This positions it alongside conformal prediction \citep{angelopoulos2023conformal} and distribution-free risk control \citep{bates2021distribution} as a fully black-box, training-free path to uncertainty in agentic systems.

\paragraph{Limitations and future work.}
Our analysis primarily uses HotpotQA with a 3-tool action space; the SWE-bench cross-benchmark validation provides evidence of generalizability across action spaces and repositories. The temperature ablation uses a smaller sample (20 questions). We cannot establish causality from observational data. Gemini 3 Pro required a \texttt{top\_p} workaround to achieve stochastic sampling (Appendix~\ref{sec:gemini_note}). The consistent-wrong annotation for Hard-stratum tasks was performed by a single annotator; all cases and classification criteria are provided for reproducibility. SWE-bench resolution rates are available only for the original 10-task astropy subset; the 50-task expanded validation uses behavioral metrics (CV, step counts) without per-task resolution data. We plan extensions to multimodal agents and analytical reasoning (FinQA, GAIA).

\section{Conclusion}

Behavioral consistency in LLM agents is measurable, associated with correctness (effect sizes that survive difficulty controls), and practically exploitable as both a runtime failure signal (AUROC 0.62--0.78) and a selective-prediction filter (87--88\% accuracy at 54--62\% coverage). With single-run evaluations misranking models 29.3\% of the time, these findings argue for multi-run evaluation, consistency-based selective prediction, and runtime consistency monitoring as standard practice. Code and data will be released upon publication.

\section*{Impact Statement}

We demonstrate that behavioral consistency provides a training-free, model-agnostic uncertainty signal for LLM agents. Our selective prediction results (87--88\% accuracy at 54--62\% coverage) and failure detection (AUROC 0.62--0.78) suggest practical deployment value for monitoring and routing in agentic systems. Our finding that single-run evaluations produce incorrect model rankings 29\% of the time has implications for benchmark reliability. One risk of consistency-based filtering is that it may systematically abstain on underrepresented task subgroups (e.g., rare question types or minority-language inputs), masking failures that disproportionately affect certain users; deployments should pair behavioral filtering with disaggregated evaluation across task demographics.

\bibliography{references}
\bibliographystyle{icml2026}

\appendix

\section{Gemini 3 Pro: Sampling Configuration}
\label{sec:gemini_note}

Gemini 3 Pro on our cloud API endpoint initially exhibited near-deterministic outputs at $T\!=\!0.7$, producing identical action sequences in $>$99\% of runs. We determined that the endpoint does not fully respect the \texttt{temperature} parameter without explicitly setting \texttt{top\_p} $< 1.0$. After adding \texttt{top\_p = 0.95}, outputs became properly stochastic (3.2 unique sequences on average, comparable to GPT-5). All Gemini results, including temperature ablation, use the corrected stochastic configuration.

\section{Temperature Ablation: Full Results}
\label{app:temp}

\noindent
\begin{minipage}{\columnwidth}
\centering
\small
\captionof{table}{Full temperature ablation (20 questions, 5 runs each). $T\!=\!0.7$ uses same questions from the main experiment.}
\begin{tabular}{llcc}
\toprule
\textbf{Model} & \textbf{Temp} & \textbf{Accuracy} & \textbf{Unique Seqs} \\
\midrule
\multirow{3}{*}{Llama 3.1 70B} & 0.0 & 82.0\% & 2.3 \\
& 0.3 & 83.0\% & 2.1 \\
& 0.7 & 82.0\% & 3.3 \\
\midrule
\multirow{3}{*}{Claude Sonnet 4.5} & 0.0 & 85.0\% & 1.0 \\
& 0.3 & 87.0\% & 1.6 \\
& 0.7 & 85.0\% & 1.7 \\
\midrule
\multirow{3}{*}{GPT-5} & 0.0 & 85.0\% & 1.4 \\
& 0.3 & 87.0\% & 1.5 \\
& 0.7 & 87.0\% & 1.6 \\
\midrule
\multirow{3}{*}{Gemini 3 Pro} & 0.0 & 75.0\% & 1.0 \\
& 0.3 & 76.0\% & 2.0 \\
& 0.7 & 75.0\% & 2.2 \\
\bottomrule
\end{tabular}
\end{minipage}

\section{Difficulty Stratification: Full Results}
\label{app:difficulty}

Task difficulty is computed as mean correctness across all runs from all four models. Strata: Easy ($\geq$0.80), Medium (0.40--0.79), Hard ($<$0.40). Of 200 questions: 132 Easy, 32 Medium, 36 Hard.

\begin{table}[h]
\centering
\small
 \caption{Consistency--correctness gap (pp) within difficulty strata. $n_c$/$n_i$ = number of consistent/inconsistent tasks per model in that stratum.}
\small
\begin{tabular}{lcccc}
\toprule
\textbf{Stratum} & \textbf{Llama} & \textbf{GPT-5} & \textbf{Claude} & \textbf{Gemini} \\
\midrule
Easy (132) & $+$11.2pp & $+$6.8pp & $-$0.5pp & $+$16.5pp \\
 & {\scriptsize 46/58} & {\scriptsize 106/8} & {\scriptsize 119/9} & {\scriptsize 100/23} \\
Medium (32) & $+$6.7pp & $+$15.0pp & $+$6.4pp & $+$34.3pp \\
 & {\scriptsize 6/22} & {\scriptsize 16/10} & {\scriptsize 16/11} & {\scriptsize 9/21} \\
Hard (36) & $-$5.6pp & $+$8.9pp & $-$5.8pp & $-$4.2pp \\
 & {\scriptsize 9/24} & {\scriptsize 19/11} & {\scriptsize 20/12} & {\scriptsize 13/22} \\
\bottomrule
\end{tabular}
\end{table}

Partial correlations between unique sequences and accuracy, controlling for difficulty (continuous): Gemini $r=-0.487$ ($p<0.001$), GPT-5 $r=-0.159$ ($p=0.025$), Llama $r=-0.156$ ($p=0.028$), Claude $r=0.011$ ($p=0.877$). The association remains significant for three of four models.

\paragraph{Consistent-Wrong analysis.} The Hard stratum reversal (negative gaps above) is explained by \emph{consistent-wrong} behavior: on tasks where the model lacks knowledge, consistency reflects confident commitment to an incorrect answer rather than reliable competence. We classify each model$\times$Hard-task pair into four categories based on unique action sequences and correctness:

\begin{itemize}[nosep,leftmargin=*]
  \item \textbf{Consistent-Correct} ($\leq$2 seqs, majority correct): 9/144 (6.3\%)
  \item \textbf{Consistent-Wrong} ($\leq$2 seqs, majority wrong): 52/144 (36.1\%)
  \item \textbf{Inconsistent-Lucky} ($\geq$4 seqs, $\geq$1 correct run): 27/144 (18.8\%)
  \item \textbf{Inconsistent-Wrong} ($\geq$4 seqs, 0 correct): 42/144 (29.2\%)
\end{itemize}

Among consistent Hard tasks, 85.2\% (52/61) are Consistent-Wrong: the model locks into one incorrect reasoning path. Of these, approximately 30\% are false negatives under token F1 $\geq$ 0.5: the model produced a semantically correct answer in a different surface form (e.g., ``Pasek and Paul'' for gold ``Pasek \& Paul'', ``Kelly Osbourne'' for gold ``Kelly Lee Osbourne''). The remaining $\sim$70\% are genuinely incorrect. When recomputed under token F1 $\geq$ 0.5, the Hard reversal disappears: all four models show non-negative consistency gaps ($+$9 to $+$16pp; Table~\ref{tab:f1_hard}), confirming that the consistency--correctness association holds under a more permissive metric. Conversely, 39.1\% (27/69) of inconsistent Hard tasks achieve at least one correct run, confirming that behavioral variability enables exploration that occasionally recovers the right answer. Per-model consistent-wrong rates (fuzzy): Llama 100\% (9/9), GPT-5 79\% (15/19), Claude 80\% (16/20), Gemini 92\% (12/13).

\begin{table}[h]
\centering
\small
\caption{Hard stratum consistency gap (pp) under token F1 $\geq$ 0.5. All gaps are non-negative, confirming the fuzzy-match reversal is a metric artifact.}
\label{tab:f1_hard}
\scriptsize
\begin{tabular}{lcccc}
\toprule
& \textbf{Llama} & \textbf{GPT-5} & \textbf{Claude} & \textbf{Gemini} \\
\midrule
Gap (pp) & $+$8.9 & $+$16.1 & $+$9.7 & $+$11.5 \\
$n_c$/$n_i$ & 9/24 & 19/11 & 20/12 & 13/22 \\
\bottomrule
\end{tabular}
\end{table}

\section{Cross-Benchmark Comparison}
\label{app:swebench}

\begin{table}[h]
\centering
\small
\caption{HotpotQA vs.\ SWE-bench. Consistency ranking preserved across benchmark complexities, action spaces, and repositories.}
\begin{tabular}{lcc}
\toprule
& \textbf{HotpotQA} & \textbf{SWE-bench} \\
\midrule
Action space & 3 tools & Unrestricted bash \\
Tasks & 200 & 50 \\
Repositories & 1 (Wikipedia) & 5 (Python OSS) \\
Avg trajectory & 3--8 steps & 10--82 steps \\
Best model CV & 8.2\% (Claude) & 17.7\% (Claude) \\
Worst model CV & 22.1\% (Llama) & 68.9\% (Llama) \\
Models tested & 4 & 4 \\
Ranking preserved & \multicolumn{2}{c}{$\checkmark$ (Claude $<$ GPT-5 $<$ Gemini $<$ Llama)} \\
\bottomrule
\end{tabular}
\end{table}

\section{GPT-4o: Additional Validation on 100-Question Subset}
\label{app:gpt4o}

We additionally evaluated GPT-4o (OpenAI) on a 100-question subset (questions 1--100) collected when the model was available via our API provider. Results are consistent with the main findings: GPT-4o shows 2.5 unique sequences per task (comparable to GPT-5's 2.4), 73.3\% accuracy, and a consistency--correctness gap of 35.9pp ($r=-0.47$, $p<0.001$). Failure detection achieves AUROC 0.661, and the path length correlation is $\rho=0.43$ ($p<0.001$). Temperature ablation shows GPT-4o gains $+$9pp accuracy at $T\!=\!0.0$, the largest temperature effect observed across all models tested. These results confirm that the consistency hierarchy extends to earlier-generation frontier models and that our findings are not specific to the four primary models in the main text.

\section{Question Type Analysis}
\label{app:qtype}

On the first 100 questions with type annotations, we compare bridge questions (multi-hop, $n=79$) with comparison questions (yes/no, $n=21$) using Llama 3.1 70B. Comparison questions show \emph{higher} correctness (80.0\% vs.\ 75.7\%) but \emph{lower} answer consistency (62.4\% vs.\ 76.6\%) and lower step variance (41\% vs.\ 63\%), highlighting that answer consistency and explanation consistency are distinct dimensions.

\section{Metric Sensitivity}
\label{app:metrics}

We verify robustness to the choice of correctness metric. Under exact match (EM), overall accuracy decreases (e.g., Claude: 43.6\% EM vs.\ 81.5\% fuzzy) but model rankings and consistency gaps are preserved. Under token F1 $>$ 0.5, results closely track fuzzy match. The consistency--correctness association holds across all three metrics for all four models.

\section{Majority Voting and Selective Prediction: Full Results}
\label{app:vote}

\paragraph{Task-outcome taxonomy.} Table~\ref{tab:taxonomy} classifies each task by consistency ($\leq$2 vs.\ $\geq$4 unique sequences) and correctness (majority answer matches gold). The taxonomy explains why filtering outperforms voting: consistent-wrong tasks (5.5--10\%) set the ceiling on filtering, while inconsistent-correct tasks (8--38\%), where voting could theoretically consolidate the right answer, are rare except for Llama.

\begin{table}[h]
\centering
\small
\caption{Task-outcome taxonomy (\% of tasks). Consistent: $\leq$2 unique seqs; Inconsistent: $\geq$4. Rows sum to $<$100\% because tasks with exactly 3 unique sequences (6--17\% per model) fall between thresholds.}
\label{tab:taxonomy}
\scriptsize
\begin{tabular}{l@{\hskip 4pt}c@{\hskip 4pt}c@{\hskip 4pt}c@{\hskip 4pt}c}
\toprule
\textbf{Model} & \textbf{Cons-Right} & \textbf{Cons-Wrong} & \textbf{Inc-Right} & \textbf{Inc-Wrong} \\
\midrule
Claude Son.\ 4.5 & 67.5 & 10.0 & 10.5 & 5.5 \\
GPT-5 & 61.5 & 9.0 & 7.5 & 7.0 \\
Gemini 3 Pro & 52.0 & 9.0 & 14.0 & 19.0 \\
Llama 70B & 25.0 & 5.5 & 38.0 & 14.0 \\
\bottomrule
\end{tabular}
\end{table}

Table~\ref{tab:vote_full} reports majority-vote accuracy across all budget levels. Gains are largest for Llama ($+$4.4pp at $k\!=\!10$) and Claude ($+$2.0pp), while Gemini shows no improvement, consistent with the hypothesis that voting helps only when the model explores diverse-but-sometimes-correct paths.

\begin{table}[h]
\centering
\small
\caption{Majority-vote accuracy (\%) by budget $k$ (500 bootstrap iterations). Gain is relative to $k\!=\!1$.}
\label{tab:vote_full}
\scriptsize
\begin{tabular}{l@{\hskip 3pt}c@{\hskip 3pt}c@{\hskip 3pt}c@{\hskip 3pt}c@{\hskip 3pt}c}
\toprule
\textbf{Model} & $k\!=\!1$ & $k\!=\!3$ & $k\!=\!5$ & $k\!=\!7$ & $k\!=\!10$ \\
\midrule
Llama 70B & 73.6 & 75.6{\tiny+2.0} & 76.9{\tiny+3.3} & 77.2{\tiny+3.6} & 78.0{\tiny+4.4} \\
Claude S.\ 4.5 & 81.5 & 82.2{\tiny+0.7} & 82.6{\tiny+1.0} & 82.9{\tiny+1.4} & 83.5{\tiny+2.0} \\
GPT-5 & 79.6 & 80.1{\tiny+0.5} & 80.2{\tiny+0.6} & 80.3{\tiny+0.7} & 80.5{\tiny+0.9} \\
Gemini 3 Pro & 72.2 & 72.3{\tiny$+$0.1} & 72.2{\tiny$-$0.0} & 72.2{\tiny$-$0.0} & 71.7{\tiny$-$0.4} \\
\bottomrule
\end{tabular}
\end{table}

Table~\ref{tab:selective_full} reports selective prediction at multiple agreement thresholds. Stricter thresholds yield higher accuracy but lower coverage. The accuracy--coverage tradeoff is favorable for Llama, Claude, and GPT-5; Gemini shows minimal improvement under selective prediction (+1.2pp at 3/3 unanimity), consistent with its high consistent-wrong rate relative to inconsistent-lucky rate.

\begin{table}[h]
\centering
\small
\caption{Selective prediction: accuracy (\%) and coverage (\%) at varying agreement thresholds.}
\label{tab:selective_full}
\scriptsize
\begin{tabular}{llcc}
\toprule
\textbf{Model} & \textbf{Threshold} & \textbf{Acc.} & \textbf{Cov.} \\
\midrule
\multirow{5}{*}{Llama 70B}
& Unanimous 3/3 & 87.2 & 53.8 \\
& Majority 2/3 & 82.2 & 82.5 \\
& Unanimous 5/5 & 88.7 & 44.2 \\
& 4/5 & 86.7 & 63.8 \\
& Majority 3/5 & 83.6 & 80.3 \\
\midrule
\multirow{5}{*}{Claude Son.\ 4.5}
& Unanimous 3/3 & 87.8 & 61.8 \\
& Majority 2/3 & 84.4 & 84.4 \\
& Unanimous 5/5 & 88.4 & 55.4 \\
& 4/5 & 87.8 & 67.3 \\
& Majority 3/5 & 85.2 & 82.5 \\
\midrule
\multirow{5}{*}{GPT-5}
& Unanimous 3/3 & 88.1 & 55.2 \\
& Majority 2/3 & 83.2 & 78.8 \\
& Unanimous 5/5 & 90.2 & 48.1 \\
& 4/5 & 86.9 & 62.3 \\
& Majority 3/5 & 84.1 & 75.3 \\
\midrule
\multirow{5}{*}{Gemini 3 Pro}
& Unanimous 3/3 & 73.4 & 64.7 \\
& Majority 2/3 & 71.6 & 85.4 \\
& Unanimous 5/5 & 74.3 & 58.4 \\
& 4/5 & 72.4 & 71.6 \\
& Majority 3/5 & 71.6 & 83.2 \\
\bottomrule
\end{tabular}
\end{table}

Table~\ref{tab:stratum_vote} shows how the $k\!=\!3$ majority-vote gain varies by difficulty stratum. Gains concentrate on Easy and Medium tasks; Hard tasks show no improvement, consistent with the finding that consistency on hard tasks reflects consistent-wrong behavior (Appendix~\ref{app:difficulty}).

\paragraph{Why self-consistency works less in agentic settings.} \citet{wang2022self} report 5--17pp gains from majority voting on single-turn CoT tasks (arithmetic, commonsense). Our agentic setting shows only 0--2pp gains. The difference is structural: in single-turn CoT, errors are primarily in the final reasoning step and the answer space is small, so diverse chains often converge on the correct answer. In multi-step agentic settings, an early wrong action (e.g., searching for the wrong entity at step~2) commits the agent to an incorrect trajectory for all subsequent steps, producing systematic errors that voting cannot correct. This makes consistency more valuable as a \emph{diagnostic} signal (is the agent confident?) than as an \emph{ensembling} strategy (pick the best answer).

\begin{table}[h]
\centering
\small
\caption{Per-stratum $k\!=\!3$ majority-vote gain (pp) over $k\!=\!1$.}
\label{tab:stratum_vote}
\scriptsize
\begin{tabular}{lcccc}
\toprule
\textbf{Stratum} & \textbf{Llama} & \textbf{Claude} & \textbf{GPT-5} & \textbf{Gemini} \\
\midrule
Easy & $+$2.4 & $+$0.3 & $+$0.7 & $+$0.7 \\
Medium & $+$2.4 & $+$1.8 & $+$1.8 & $-$1.2 \\
Hard & $-$0.0 & $+$0.9 & $-$1.4 & $-$0.5 \\
\bottomrule
\end{tabular}
\end{table}

\section{Temperature Ablation}
\label{app:temp_ablation_fig}

Figure~\ref{fig:temp} shows temperature ablation results across all four models. All models show monotonically increasing unique sequences with temperature. Claude and Gemini achieve near-perfect consistency at $T\!=\!0.0$ (1.0 sequences) while Llama retains 2.3, suggesting that model architecture matters beyond sampling noise. The finding that $T\!=\!0.0$ does not eliminate divergence for Llama or GPT-5 indicates that agentic inconsistency arises from factors beyond token-level sampling.

\begin{figure}[h]
\centering
\includegraphics[width=\columnwidth]{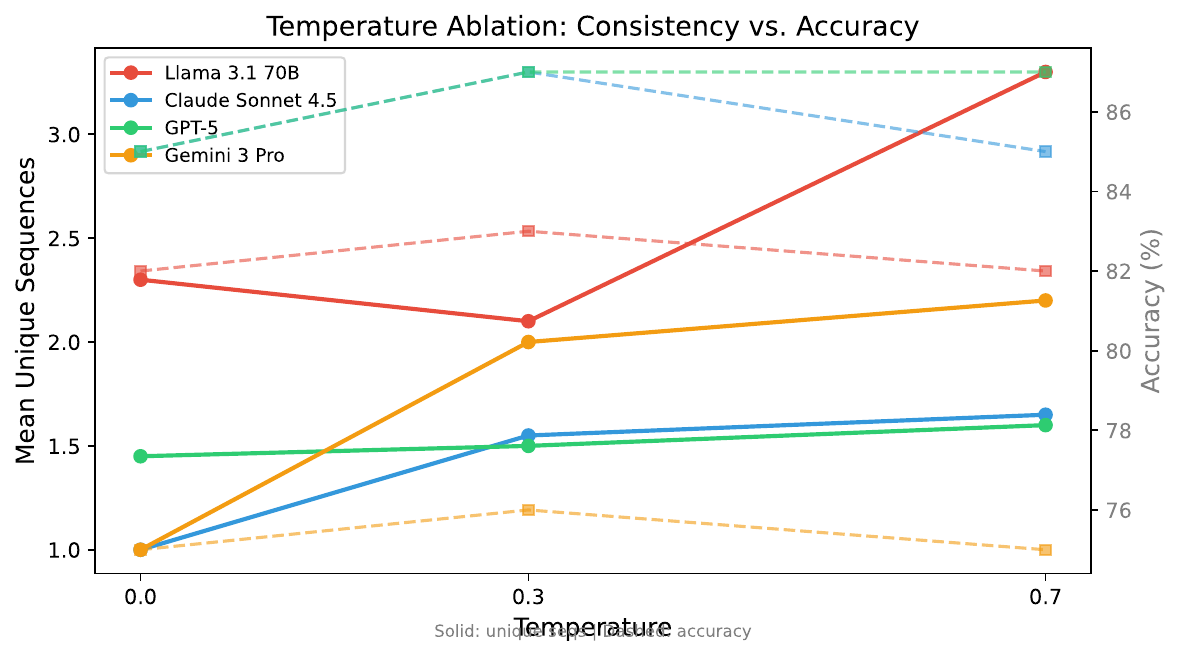}
\caption{Temperature ablation (20 questions, 5 runs). Solid: unique sequences (left axis); dashed: accuracy (right axis).}
\label{fig:temp}
\end{figure}

\end{document}